%% file: aaai2027-lens-arxiv-v1.tex
\documentclass[letterpaper]{article}
\usepackage[preprint]{aaai2027}
\usepackage[hyphens]{url}
\usepackage{graphicx}
\usepackage{natbib}
\usepackage{caption}
\usepackage{algorithm}
\usepackage{algorithmic}
\usepackage{amsmath}
\usepackage{amssymb}
\usepackage{amsthm}
\usepackage{booktabs}
\usepackage{multirow}

\newtheorem{definition}{Definition}
\newtheorem{proposition}{Proposition}

\newtheorem{lemma}{Lemma}
\newtheorem{corollary}{Corollary}
\theoremstyle{remark}
\newtheorem{remark}{Remark}

\begin{document}

\title{LENS: In-Context Search via Latent Evidence Exploration over Dynamic Raw Documents}

\author{
Xingjun Wang,
Gongsheng Li,
Qi Fan,
Yunlin Mao,
Luyan Su,
Yingda Chen$^{\dagger}$
}
\affiliations{
ModelScope Team, Alibaba Group\\
Hangzhou, China\\
\{xingjun.wxj, ligongsheng.lgs, luoqifan.fq, maoyunlin.myl, suluyan.sly, yingda.chen\}@alibaba-inc.com
}

\maketitle
\begingroup
\renewcommand{\thefootnote}{\fnsymbol{footnote}}
\footnotetext[2]{Corresponding author.}
\endgroup

\begin{abstract}
Large language model agents increasingly need to answer questions over dynamic raw-document collections, where files may be added or updated before evidence can be preprocessed into fixed representations. Relevant evidence may appear as spans, sections, pages, or tables whose usefulness is query-dependent. Existing retrieval-augmented approaches typically materialize the evidence space before querying through fixed chunking, embeddings, or persistent indexes. While effective for lookup, such representations impose preprocessing cost, can become stale after document updates, and commit to an evidence granularity before the query is known.

We formulate in-context search as Budgeted Evidence Localization over a latent evidence space induced by dynamic raw documents, and propose Latent Evidence Exploration and Search (LENS) as an index-free framework for this setting. Rather than pre-materializing the full evidence space, LENS maintains a query-conditioned belief over candidate evidence units, iteratively selecting candidates through complementary lexical, local, and exploratory proposal policies, updating the belief with observations from an LLM-based relevance oracle, and narrowing the search toward high-posterior evidence regions under a controllable budget. The resulting evidence is consolidated into compact, source-grounded regions of interest and further compressed into self-organizing knowledge clusters for reuse across semantically related queries.

On a controlled 500-question evaluation with matched corpus snapshots, LENS achieves 62.4\% exact match and 84.8\% evidence recall, while a ReAct-style iterative baseline achieves 65.2\% exact match but only 50.4\% evidence recall. Across controlled scales, LENS consistently provides the strongest supporting-fact localization and answer grounding. On a fixed 150-question fullwiki subset over the raw Wikipedia dump with zero indexing, LENS and ReAct are nearly tied in official answer quality (43.3\% vs. 42.7\% EM), while LENS grounds a larger share of answers in retrieved evidence (84.0\% vs. 70.7\%). A no-retrieval Closed-Book reference highlights the contribution of model memory and motivates reporting retrieval gains relative to this baseline. LENS remains query-ready immediately after corpus changes, requires no preprocessing or persistent index, and preserves source-grounded evidence localization throughout.
\end{abstract}

\section{Introduction}
\label{sec:intro}

Large language model (LLM) agents increasingly operate over collections of raw documents that evolve faster than evidence can be reliably preprocessed into fixed representations. In such settings, relevant evidence is not a stable object known before the query. It may be a paragraph span, a table entry, a section, a page, or a cross-document chain whose appropriate granularity depends on both the question and the current document state. This makes document-grounded question answering different from retrieval over a static corpus of pre-segmented passages.

Existing retrieval-augmented approaches typically address document question answering by materializing the evidence space before querying, through chunking, dense embeddings, summaries, persistent sparse indexes, or graph-like memory structures~\cite{lewis2020retrieval,karpukhin2020dense,izacard2022atlas}. These representations are effective for static lookup, especially when the corpus is stable and preprocessing cost can be amortized. However, in dynamic raw-document collections, pre-materialization introduces a trade-off: setup and update costs must be paid before querying, indexes can become stale after document changes, and fixed chunks commit to an evidence granularity before the query reveals what evidence is needed.

The central difficulty is that the evidence space induced by raw documents is latent, variable-boundary, dynamic, and structured. It is latent because answer-bearing evidence exists in the documents but is not known in advance; variable-boundary because the useful evidence windows are not limited to a fixed set of chunks, which makes the space discrete but combinatorially large; dynamic because document updates change the space itself; and structured because lexical, layout, path, and historical signals induce non-uniform priors over likely evidence regions. Treating this space as a fixed finite collection of chunks can therefore obscure the actual search problem faced by an LLM agent.

We formulate this setting as \emph{Budgeted Evidence Localization} over a \emph{latent evidence space} induced by dynamic raw documents. The goal is not merely to return the nearest precomputed passages, but to infer a compact source-grounded evidence set under token, latency, and oracle-call constraints. This formulation makes the cost of using an LLM as a relevance oracle explicit, while preserving the query-conditioned nature of evidence granularity and document freshness.

To make this formulation practical, we propose \textbf{Latent Evidence Exploration and Search (LENS)}. LENS first forms a low-cost prior over candidate evidence regions using document signals available before any expensive oracle interaction. It then performs sequential exploration: each observation from an LLM-based relevance oracle updates a belief over candidate evidence units and guides subsequent exploration toward regions with higher expected utility. Finally, LENS consolidates selected evidence into compact, source-grounded regions and organizes confirmed evidence for reuse across semantically related follow-up queries.

We evaluate LENS under a dynamic raw-corpus protocol built around nested question sets and matched raw-document snapshots. This protocol measures answer quality, evidence localization, freshness under corpus growth, and query-time budget. All systems are compared under identical sampled questions and corpus boundaries, allowing paired analysis while keeping the raw-corpus setting auditable.

Our contributions are as follows. First, we formulate in-context search over evolving raw-document collections as Budgeted Evidence Localization over a latent evidence space. Second, we introduce LENS, an index-free sequential exploration framework that combines low-cost priors, oracle-guided evidence refinement, budget-aware stopping, and source-grounded consolidation. Third, we design a dynamic raw-corpus evaluation protocol that jointly reports answer quality, evidence localization, freshness, budget-normalized quality, and no-retrieval reference scores. Finally, we provide mechanism evidence through source-grounding diagnostics and an ablation that isolates multi-signal prior formation from sequential exploration, rather than relying only on aggregate answer scores.

\section{Background}
\label{sec:background}

We study document collections whose contents and boundaries may change between queries. Files can be added, updated, or removed, and the evidence needed by a query may be located in spans, tables, pages, or cross-document relations. This setting differs from static open-domain retrieval because both the corpus state and the appropriate evidence granularity are query-dependent.

Retrieval-augmented generation typically constructs a finite representation of the evidence space before a query arrives. Chunking, embedding indexes, sparse indexes, summary trees, and document graphs all instantiate this strategy. These representations are useful when the corpus is stable, but they commit to a fixed representation before the query is known and must be rebuilt or updated when the underlying documents change.

LLMs provide strong semantic relevance judgments, but each oracle interaction consumes tokens, latency, and cost. A search method over dynamic raw documents therefore needs to trade off immediate relevance, information gain, source traceability, and budget. LENS uses this view to cast in-context search as sequential evidence localization rather than as a one-shot top-$k$ lookup problem.

\section{Budgeted Evidence Localization}
\label{sec:bel}

We consider a collection of raw documents that evolves over time. Let
\begin{equation}
\mathcal{D}_t = \{d_1^{(t)}, d_2^{(t)}, \ldots, d_N^{(t)}\}
\end{equation}
denote the document collection at time $t$. A query $q$ arrives after the current corpus state is fixed, and the system must answer using evidence from $\mathcal{D}_t$ rather than from a stale representation of $\mathcal{D}_{t-1}$.

\begin{definition}[Latent Evidence Space]
For a dynamic raw-document collection $\mathcal{D}_t$, the latent evidence space is
\begin{equation}
\mathcal{E}_t \triangleq \{(d,s,e) \mid d \in \mathcal{D}_t,\ 0 \le s < e \le |d|\}.
\end{equation}
Each element denotes a candidate evidence window in a raw document.
\end{definition}

The space $\mathcal{E}_t$ is not explicitly enumerated by LENS. It is latent because the answer-bearing region is unknown before querying, variable-boundary because window boundaries may vary at span-level resolution, which makes $\mathcal{E}_t$ finite but combinatorially large in $|d|$, dynamic because $\mathcal{E}_t$ changes with $\mathcal{D}_t$, and structured because document paths, textual anchors, compiled summaries, and prior successful searches induce non-uniform beliefs over the space.

Queries impose different evidence requirements. We use an intent variable
\begin{equation}
\begin{aligned}
I(q) \in \{&\text{lookup},\text{ computation},\\ &\text{comparison},\text{ aggregation},\text{ summarize}\}
\end{aligned}
\end{equation}
to indicate whether a query can be answered by a localized span, requires multiple atomic facts, or requires a higher-level synthesis. The intent induces a set of data requirements
\begin{equation}
\mathcal{D}_{\mathrm{req}}(q,I)=\{f_1,\ldots,f_K\},
\end{equation}
a decomposition of the query into atomic facts together with optional transformation rules that combine them. Lookup queries have $K=1$, whereas comparison, computation, and aggregation queries have $K>1$ and their facts may reside in different documents.

Evidence localization is therefore defined per atomic fact rather than per query.

\begin{definition}[Per-Fact Evidence Target]
\label{def:target}
For each $f_j \in \mathcal{D}_{\mathrm{req}}(q,I)$, let $Z_j^* \in \mathcal{E}_t$ be a minimal sufficient evidence window for $f_j$. The evidence target of the query is the collection $\mathbf{Z}^*=\{Z_1^*,\ldots,Z_K^*\}$, and localization is complete only when every requirement is covered.
\end{definition}

This per-fact formulation is what makes the target well defined on multi-hop queries. A comparison between two entities needs one fact from each of two documents, so no single contiguous window is sufficient and a query-level single-window target would not exist. Defining the target per fact keeps the inference object a single window --- so that beliefs and stopping rules remain tractable --- while the query-level output is a set, which is also what the system returns.

For a fact $f_j$, the ideal inference target is the posterior
\begin{equation}
\begin{aligned}
P(Z_j^* \mid f_j,q,\mathcal{H}_t)
&\propto P(\mathcal{H}_t \mid Z_j^*, f_j, q)\\
&\quad \cdot \pi_{\mathrm{prior}}(Z_j^* \mid f_j,q,\mathcal{D}_t).
\end{aligned}
\end{equation}
where $\pi_{\mathrm{prior}}$ is a query-conditioned initial belief and $\mathcal{H}_t$ is the history of oracle observations. Because the prior already conditions on the query, evidence about $Z_j^*$ accrues through the observation likelihood. This likelihood is not available in closed form, so LENS treats an LLM as a costly relevance oracle and approximates posterior concentration through sequential observations.

A budget $B$ limits the number of oracle calls, tokens, or wall-clock time available for a query. The output is a pair $(E^*, a)$, where $E^*$ is a compact, source-grounded evidence set that covers the localized windows $\{\hat Z_j\}_{j\le K}$ and $a$ is the synthesized answer. This distinguishes evidence localization from pure answer generation: a correct answer without traceable evidence is insufficient for the setting considered here.

\section{The LENS Algorithm}
\label{sec:algorithm}

\begin{figure*}[t]
\centering
\includegraphics[width=\textwidth]{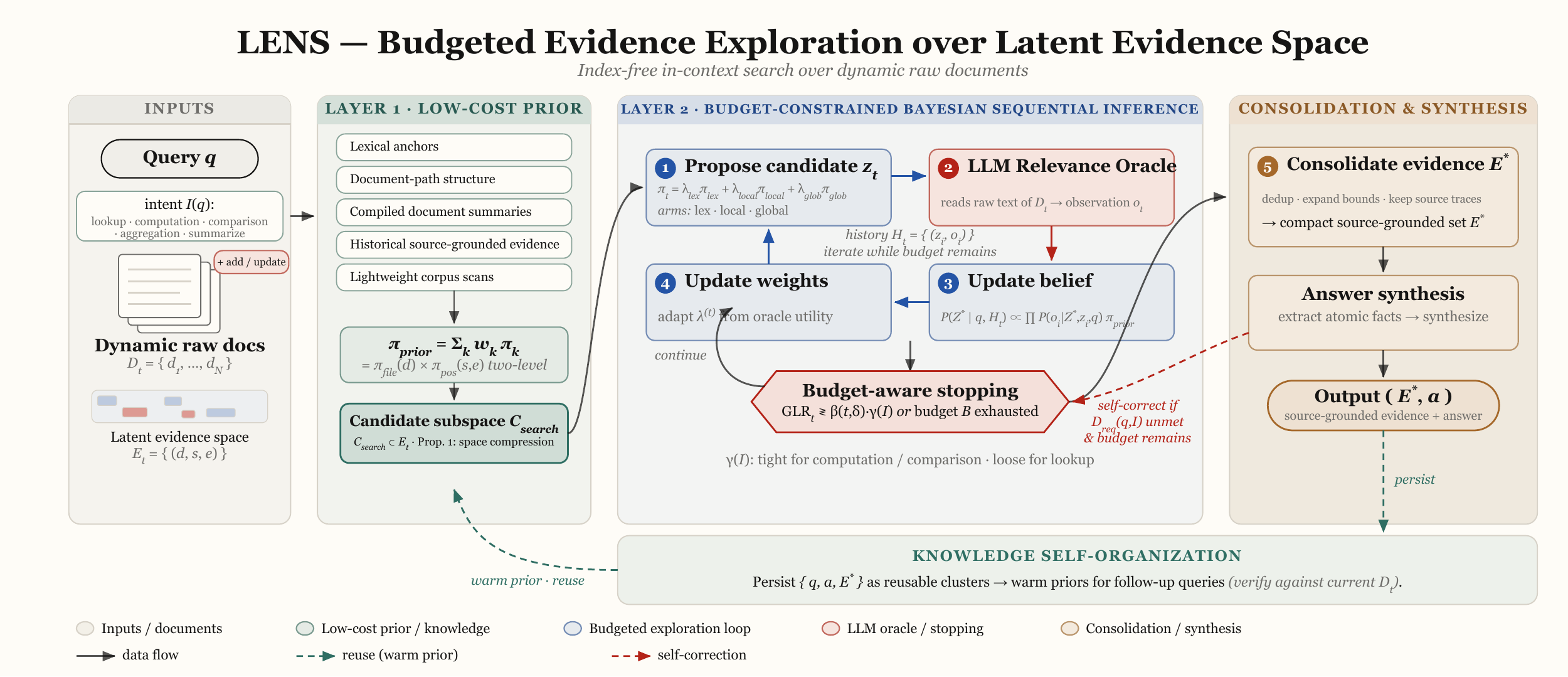}
\caption{Overall LENS framework. LENS forms a query-conditioned prior over candidate evidence regions, runs a budget-constrained propose--observe--update loop, and consolidates selected regions into a compact source-grounded evidence set for answer synthesis. It never pre-materializes a persistent index over the raw-document collection.}
\label{fig:lens_pipeline}
\end{figure*}

Figure~\ref{fig:lens_pipeline} presents the overall framework. LENS first constructs a low-cost prior over the latent evidence space, then performs budget-constrained sequential inference in a propose--observe--update loop. Finally, it consolidates high-belief regions into a compact source-grounded evidence set for answer synthesis. This separation is central: prior formation cheaply narrows the search domain, while sequential refinement spends oracle calls only where they are informative.

\subsection{Layer 1: Low-Cost Prior over Latent Evidence}

Directly exploring $\mathcal{E}_t$ is infeasible. As shown on the left of Figure~\ref{fig:lens_pipeline}, LENS first compresses the search space by fusing five families of signals that are available without reading the corpus through an LLM: lexical anchors, document-path structure, compiled document summaries when available, historical source-grounded evidence from prior successful searches, and lightweight corpus scans. Rather than treating such signals as independent retrieval modules, LENS uses them as approximations to a prior
\begin{equation}
\pi_{\text{prior}}(z \mid q, \mathcal{D}_t) \approx \sum_{k \in \mathcal{K}_0} w_k\,\pi_k(z \mid q, \mathcal{D}_t),
\end{equation}
where $\mathcal{K}_0$ denotes a family of low-cost proposal signals. This joint prior factors, by the chain rule, into a marginal--conditional pair that separates document selection from within-document localization:
\begin{equation}
\pi_{\text{prior}}(z \mid q,\mathcal{D}_t) = \pi_{\text{file}}(d_z \mid q,\mathcal{D}_t)\,\pi_{\text{pos}}(s_z,e_z \mid d_z,q),
\end{equation}
where $d_z$ is the document associated with region $z$. The two equations therefore describe the same object: the mixture in the first specifies how the signals are fused, and the factorization in the second specifies at which level each factor acts. Making the decomposition explicit matters because strong file-level evidence does not automatically imply precise within-document localization.

\begin{proposition}[Low-cost space compression]
\label{prop:compression}
Let $\mathcal{C}_{\mathrm{init}}$ be the finite candidate document set induced by low-cost prior signals. Subsequent evidence exploration is restricted from $\mathcal{E}_t$ to $\mathcal{C}_{\mathrm{search}}=\{(d,s,e)\mid d\in\mathcal{C}_{\mathrm{init}},0\le s<e\le |d|\}$. Thus, before any iterative oracle budget is consumed, LENS reduces the effective search domain from the full raw-document collection to a query-conditioned subspace.
\end{proposition}

\subsection{Layer 2: Budget-Constrained Sequential Inference}

After prior formation, LENS enters the budgeted exploration loop shown at the center of Figure~\ref{fig:lens_pipeline}, which cycles through four steps while budget remains and some requirement in $\mathcal{D}_{\mathrm{req}}(q,I)$ is still uncovered: (i) propose a candidate evidence region $z_t$, (ii) query the LLM relevance oracle on the raw text of the region to obtain an observation $o_t$, (iii) update the beliefs over evidence regions, and (iv) adapt proposal weights and coverage estimates. Let $\mathcal{H}_t=\{(z_i,o_i)\}_{i=1}^t$ denote the observation history accumulated by this loop. Each observation is informative about every outstanding requirement, so LENS maintains one belief per fact and updates all of them from the shared history:
\begin{equation}
\begin{aligned}
P(Z_j^* \mid f_j,q, \mathcal{H}_t)
&\propto \prod_{i=1}^t P(o_i \mid Z_j^*, z_i, f_j, q)\\
&\quad \cdot \pi_{\mathrm{prior}}(Z_j^* \mid f_j,q, \mathcal{D}_t).
\end{aligned}
\end{equation}
A single oracle call therefore serves all facts at once, which is why the per-fact formulation does not multiply the oracle budget by $K$.

The next candidate should balance exploitation and exploration: it should use the current posteriors to refine promising regions, while retaining the ability to discover evidence missed by lexical or structural priors. An ideal information-directed objective~\cite{russo2014ids} selects
\begin{equation}
\begin{aligned}
z_{t+1} &= \arg\min_{z \in \mathcal{C}_{\mathrm{search}}} \Psi_t(z),\\
\Psi_t(z) &= \frac{[\Delta_t(z)]^2}{\mathbb{I}(\mathbf{Z}^*;O_z \mid q,\mathcal{H}_t)}.
\end{aligned}
\end{equation}
where $\Delta_t(z)$ is the expected immediate relevance gap and the denominator is the expected information gain about the outstanding targets $\mathbf{Z}^*$. Minimizing this information ratio is principled: it targets the trade-off between immediate relevance and long-run information gain that underlies regret-optimal sequential selection, so that oracle budget is spent where it is most informative about $\mathbf{Z}^*$. Exact computation is intractable in raw documents, so LENS approximates this criterion with complementary proposal families:
\begin{equation}
\begin{aligned}
\pi_t(z)=&\ \lambda_{\text{lex}}^{(t)}\pi_{\text{lex}}(z)+\lambda_{\text{local}}^{(t)}\pi_{\text{local}}(z)\\
&+\lambda_{\text{global}}^{(t)}\pi_{\text{global}}(z).
\end{aligned}
\end{equation}
Lexical proposals exploit anchors, local proposals refine around high-belief regions, and global proposals guard against semantic omissions. Treating each proposal family as an arm, LENS adapts the mixture weights $\lambda^{(t)}$ online from observed oracle utility, closing the propose--observe--update cycle in Figure~\ref{fig:lens_pipeline} without enumerating the full latent space.

\subsection{Budget-Aware Stopping}

The exploration loop should not continue merely because more context can be read. As depicted in Figure~\ref{fig:lens_pipeline}, each iteration terminates in a stopping decision: LENS exits the loop when either the remaining budget is insufficient or every requirement is localized with sufficiently concentrated belief. Following fixed-confidence best-arm identification ideas~\cite{garivier2016track}, a conceptual per-fact stopping statistic is
\begin{equation}
\begin{aligned}
\mathrm{GLR}_t^{(j)}
= \min_{z\ne \hat Z_{j,t}} \sum_{i\le t}
\log\frac{P(o_i\mid \hat Z_{j,t},f_j,q)}{P(o_i\mid z,f_j,q)}.
\end{aligned}
\end{equation}
where $\hat Z_{j,t}$ is the current highest-belief region for fact $f_j$. LENS stops when $\min_{j\le K}\mathrm{GLR}_t^{(j)}$ exceeds an intent-modulated threshold $\beta(t,\delta)\,\gamma(I)$ --- that is, when the weakest requirement is resolved --- or when the budget is exhausted, where $\gamma(I)$ tightens the criterion for computation and comparison intents and relaxes it for lookup. A lookup query has a single requirement and can often stop with one compact region, whereas comparison and computation queries cannot stop until each of their $K$ facts has its own confirmed window, which is exactly the coverage condition on $\mathcal{D}_{\mathrm{req}}(q,I)$.

\subsection{Evidence Consolidation and Answer Synthesis}

Once the loop stops, the selected regions enter the consolidation-and-synthesis stage on the right of Figure~\ref{fig:lens_pipeline}. Consolidation merges the per-fact windows $\{\hat Z_j\}_{j\le K}$, removes redundant or overlapping regions, expands boundaries when necessary for interpretability, and preserves source traces, yielding a compact source-grounded evidence set $E^*$. Answer synthesis then operates on $E^*$: for computation and comparison queries, LENS separates extraction of atomic facts from answer synthesis; for lookup-style queries, a single-stage synthesis may be sufficient. When synthesis cannot satisfy $\mathcal{D}_{\mathrm{req}}(q,I)$ and budget remains, LENS triggers the self-correction path in Figure~\ref{fig:lens_pipeline}: a bounded step that relaxes the stopping threshold and re-enters the exploration loop with an expanded candidate set before re-synthesizing. The final output is the pair $(E^*,a)$, an answer grounded in explicit evidence regions rather than only in retrieved text snippets.

\subsection{Algorithm Summary and Theoretical Properties}

Algorithm~\ref{alg:lens} summarizes the inference loop. The algorithm is intentionally written at the method level rather than in implementation-specific terms.

\begin{algorithm}[t]
\caption{LENS: Budgeted Evidence Localization}
\label{alg:lens}
\begin{algorithmic}[1]
\STATE \textbf{Input:} query $q$, dynamic corpus $\mathcal{D}_t$, budget $B$
\STATE Build low-cost prior $\pi_{\text{prior}}(z\mid q,\mathcal{D}_t)$ over candidate evidence regions
\STATE Derive requirements $\mathcal{D}_{\mathrm{req}}(q,I)=\{f_1,\ldots,f_K\}$
\STATE Initialize observation history $\mathcal{H}_0\leftarrow \emptyset$ and candidate subspace $\mathcal{C}_{\mathrm{search}}$
\WHILE{budget remains and some $f_j$ is uncovered}
  \STATE Select proposal family and sample candidate region $z_t$
  \STATE Query relevance oracle to obtain observation $o_t$
  \STATE Update per-fact beliefs $P(Z_j^*\mid f_j,q,\mathcal{H}_t)$
  \STATE Update proposal weights and requirement coverage
\ENDWHILE
\STATE Consolidate the per-fact windows into source-grounded evidence set $E^*$
\STATE Synthesize answer $a$ from $E^*$ and persist reusable evidence clusters when appropriate
\STATE \textbf{Return:} $(E^*,a)$
\end{algorithmic}
\end{algorithm}

Under the abstraction above, LENS admits two analysis statements and a budget guarantee. First, low-cost priors reduce the search domain before any iterative oracle budget is spent. Second, when the oracle relevance signal is locally stable, successive observations progressively concentrate the belief over candidate evidence regions, so that additional budget yields diminishing exploration returns. The online cost of this process is bounded independently of corpus size.

\begin{proposition}[Bounded oracle complexity]
For a loop budget of $L$ exploration rounds, the number of LLM oracle interactions performed by LENS is bounded by $c_0+c_1 L$ for small constants $c_0,c_1$ that depend only on the configuration and not on the query ($c_0=4$, $c_1=2$ in our setting). The bound is independent of the number of requirements $K$, because one oracle observation updates all per-fact beliefs, and independent of the number of latent evidence windows induced by $\mathcal{D}_t$.
\end{proposition}

These statements are analysis guides for the method rather than tight guarantees: stricter probabilistic oracle models, position-level priors, and resampling analyses remain future work.

\section{Related Work}
\label{sec:related}

Retrieval-augmented generation methods retrieve external evidence before generation and have become a standard approach for knowledge-intensive tasks~\cite{lewis2020retrieval,karpukhin2020dense,izacard2022atlas}. Dense, sparse, and hybrid systems are effective when a stable corpus can be preprocessed into persistent representations. Hierarchical and graph-based retrieval systems further organize documents into summaries, trees, or memory graphs~\cite{sarthi2024raptor,gutierrez2024hipporag}, improving reuse and traversal when the supporting structures can be built and maintained. LENS addresses a different operating point: dynamic raw-document collections where query readiness, update cost, evidence granularity, and source traceability are part of the task rather than external deployment details.

Long-context language models provide another way to avoid a persistent index by placing large amounts of raw text directly into the model context~\cite{beltagy2020longformer,liu2024lost}. This reduces explicit index construction, but shifts cost to query time through high token consumption and may still fail to localize the specific evidence inside long inputs. Tool-using agents can search, read, and refine their context over multiple steps~\cite{yao2023react,asai2023selfrag}, but they often lack an explicit formulation of evidence localization under a controllable budget. LENS treats interaction as posterior-guided evidence exploration over a latent evidence space, and evaluates the resulting trade-off through answer quality, source-grounded evidence localization, freshness, and lifecycle cost.

\section{Experiments}
\label{sec:experiments}

\subsection{Setup}
We evaluate LENS on HotpotQA fullwiki~\cite{yang2018hotpotqa}, a multi-hop question answering benchmark where each question requires reasoning over two or more Wikipedia articles. We report results under two conditions:

\textbf{Controlled evaluation ($D_n$).} From the validation split (7{,}405 questions) we draw frozen, proportionally stratified evaluation sets over $\mathit{type}\times\mathit{difficulty}$ strata (seed~42). Our primary controlled evaluation uses $n=500$ questions paired with a matched corpus snapshot $D_{500}$ containing the gold supporting articles, context distractors, and a deterministic background pool. This is the largest statistically robust scale we evaluate and enables paired comparison across all baselines including index-dependent systems.

\textbf{Open-domain fullwiki.} The complete raw Wikipedia dump underlying the fullwiki validation split is stored as 15{,}517 JSON shards without preprocessing, chunking, or indexing. We evaluate $n=150$ fixed questions on this corpus to measure corpus-scale robustness under the zero-index constraint.

\textbf{Systems.} We compare five arms: (1)~\textbf{LENS}, the full algorithm with multi-signal prior, sequential exploration, budget-aware stopping, and evidence consolidation; (2)~\textbf{ReAct Search}, a strong iterative baseline using ReAct-style~\cite{yao2023react} tool-use reasoning with the same corpus access but without LENS's structured prior formation or budget-constrained belief updates; (3)~\textbf{Hybrid-RAG}, a retrieval-augmented baseline combining BM25 and dense embedding retrieval with a pre-materialized index; (4)~\textbf{BM25-RAG}, a sparse-retrieval baseline using a pre-built BM25 index; and (5)~\textbf{Closed-Book}, a no-retrieval reference that estimates the score attributable to model parameters alone. Systems~(3) and~(4) require pre-materialized indexes and therefore face a lifecycle limitation under corpus changes.

\textbf{Hyperparameters.} All systems share one chat backend (Qwen3.7, a 35B-parameter mixture-of-experts model with 3B active parameters, fixed temperature, 120\,s per-call timeout) under a 300\,s per-question wall-clock cap. LENS runs its DEEP configuration with a 128K query-time token budget and at most 10 candidate files admitted to evidence extraction. All runs use cold caches with no knowledge reuse.

\textbf{Environment.} Experiments run on a single Apple M4 Pro workstation (12 cores, 48\,GB RAM, macOS~15) with no local GPU; all model calls are served by a remote OpenAI-compatible endpoint, so reported latencies include network time. At most 5 questions are evaluated concurrently.

\textbf{Metrics.} EM and F1 are the official answer metrics. Ev.Rec measures retrieval of gold supporting-fact documents, while Ground measures whether the final answer is traceable to retrieved evidence. Judge is an auxiliary semantic-equivalence diagnostic from an independent model and is not used as the primary score.

\begin{table}[t]
\centering
{\footnotesize
\begin{tabular}{lrrrr}
\toprule
Snapshot & Samples & Articles & Evidence & Background \\
\midrule
$D_{125}$ & 125 & 5{,}416 & 250 & 4{,}062 \\
$D_{250}$ & 250 & 10{,}808 & 500 & 8{,}106 \\
$D_{500}$ & 500 & 21{,}424 & 996 & 16{,}068 \\
\midrule
System ($D_{500}$) & Ready & Rebuild & Index(s) & Storage \\
\midrule
LENS & yes & no & 0.0 & 0 \\
ReAct & yes & no & 0.0 & 0 \\
BM25-RAG & no & yes & 4.0 & 10.0MB \\
Hybrid-RAG & no & yes & 5.7 & 54.9MB \\
\bottomrule
\end{tabular}
}
\caption{Snapshot construction and lifecycle audit. The upper block reports deterministic $D_n$ corpus sizes; the lower block reports $D_{500}$ query readiness, rebuild requirement, index time, and index storage.}
\label{tab:lifecycle}
\end{table}

\subsection{Main Results: Controlled Evaluation}

\begin{table}[t]
\centering
{\footnotesize
\begin{tabular}{lcccc}
\toprule
System & EM & F1 & Ev.Rec & Ground \\
\midrule
ReAct Search  & \textbf{65.2} & \textbf{78.9} & 50.4          & 71.8 \\
LENS          & 62.4          & 76.9          & \textbf{84.8} & \textbf{96.8} \\
Hybrid-RAG    & 38.4          & 51.1          & 80.8          & 92.2 \\
BM25-RAG      & 28.8          & 42.3          & 71.8          & 95.8 \\
Closed-Book   & 35.2          & 47.0          & 0.0           & 0.0 \\
\bottomrule
\end{tabular}
}
\caption{Controlled evaluation on $D_{500}$ (\%, $n$=500). EM/F1 are official answer metrics; Ev.Rec is supporting-fact document recall; Ground is the percentage of answers traceable to retrieved evidence.}
\label{tab:controlled}
\end{table}

Table~\ref{tab:controlled} reports the primary controlled evaluation. ReAct Search attains the highest answer score (65.2\% EM, 78.9\% F1), while LENS remains close on answer quality (62.4\% EM, 76.9\% F1) and provides substantially stronger evidence localization. LENS achieves 84.8\% evidence recall and 96.8\% grounded answers, compared with 50.4\% and 71.8\% for ReAct. This separates two evaluation dimensions: ReAct more often produces the exact answer string, whereas LENS more reliably localizes and traces the supporting evidence. Relative to Closed-Book, LENS gains 27.2 pp EM and ReAct gains 30.0 pp, so retrieval benefits are reported against a no-retrieval reference rather than assumed from answer correctness alone.

\subsection{Open-Domain Fullwiki Results}

\begin{table}[t]
\centering
{\footnotesize
\begin{tabular}{lcccc}
\toprule
System & EM & F1 & Ev.Rec & Ground \\
\midrule
LENS          & \textbf{43.3} & \textbf{57.5} & 45.0 & \textbf{84.0} \\
ReAct Search  & 42.7          & 57.3          & \textbf{50.4} & 70.7 \\
Closed-Book   & 38.7          & 49.7          & 0.0  & 0.0 \\
\bottomrule
\end{tabular}
}
\caption{Open-domain fullwiki results (\%, $n$=150, fixed sample IDs). LENS and ReAct search the raw Wikipedia dump (15{,}517 shards) with zero indexing; Closed-Book reads no corpus.}
\label{tab:fullwiki}
\end{table}

Table~\ref{tab:fullwiki} reports results on the full raw Wikipedia corpus. LENS and ReAct are effectively tied on official answer quality (43.3\% vs. 42.7\% EM), while LENS grounds more answers in retrieved evidence (84.0\% vs. 70.7\%). Closed-Book reaches 38.7\% EM, so fullwiki retrieval gains are modest but positive: +4.6 pp for LENS and +4.0 pp for ReAct. BM25-RAG and Hybrid-RAG require separate full-dump indexing and are evaluated through the controlled/lifecycle arms.

\subsection{Evidence Recall Leadership}

\begin{table}[t]
\centering
{\footnotesize
\begin{tabular}{lccc}
\toprule
System & $D_{125}$ & $D_{250}$ & $D_{500}$ \\
\midrule
LENS          & \textbf{89.1} & \textbf{85.9} & \textbf{84.8} \\
Hybrid-RAG    & 81.3          & 77.8          & 80.8 \\
BM25-RAG      & 71.6          & 71.1          & 71.8 \\
ReAct Search  & 46.1          & 47.9          & 50.4 \\
Closed-Book   & 0.0           & 0.0           & 0.0 \\
\bottomrule
\end{tabular}
}
\caption{Evidence recall (\%) across controlled evaluation scales. LENS consistently provides the strongest supporting-fact localization, while Closed-Book has no evidence trace by design.}
\label{tab:evidence_recall}
\end{table}

Table~\ref{tab:evidence_recall} presents evidence recall across all controlled evaluation scales. LENS is the top evidence-localization system at every scale, with 84.8--89.1\% evidence recall. Its lead over ReAct is 43.0 pp on $D_{125}$, 38.0 pp on $D_{250}$, and 34.4 pp on $D_{500}$. Hybrid-RAG also retrieves substantial evidence, but its answer quality remains far below LENS and ReAct, indicating that locating candidate evidence and synthesizing the exact multi-hop answer are separable failure modes.

\subsection{Corpus Staleness and Lifecycle Robustness}

\begin{table}[t]
\centering
{\footnotesize
\begin{tabular}{lcccccc}
\toprule
 & \multicolumn{3}{c}{Exact Match (\%)} & \multicolumn{3}{c}{Evidence Recall (\%)} \\
\cmidrule(lr){2-4}\cmidrule(lr){5-7}
System & Fresh & Stale & $\Delta$ & Fresh & Stale & $\Delta$ \\
\midrule
BM25-RAG      & 30.4 & 2.4  & $-28.0$ & 70.1 & 0.0  & $-70.1$ \\
Hybrid-RAG    & 42.4 & 13.6 & $-28.8$ & 75.1 & 5.5  & $-69.6$ \\
ReAct Search  & 72.0 & 72.8 & $+0.8$  & 54.1 & 48.3 & $-5.7$ \\
LENS          & 69.6 & 67.2 & $-2.4$  & 84.7 & 83.9 & $-0.8$ \\
\bottomrule
\end{tabular}
}
\caption{Staleness robustness under corpus expansion ($D_{125}\to D_{250}$, $\Delta n=125$). Index-heavy systems reuse an index built on $D_{125}$ while answering newly added $D_{250}$ questions; index-free systems query the updated corpus directly. Negative gaps indicate stale degradation; the table measures freshness/lifecycle robustness rather than fresh-corpus answer quality.}
\label{tab:staleness}
\end{table}

The stale-index arm measures what happens when an index built on $D_{125}$ is reused after the corpus expands to $D_{250}$. BM25-RAG and Hybrid-RAG are index-dependent and therefore lose most of their ability to answer newly added questions: EM drops by 28.0 and 28.8 pp, and evidence recall drops by 70.1 and 69.6 pp. ReAct and LENS are index-free, so they can query the expanded corpus immediately; their EM changes are small (+0.8 pp and -2.4 pp), and LENS retains nearly all supporting-fact recall (84.7\% to 83.9\%). Table~\ref{tab:lifecycle} shows the corresponding lifecycle condition: index-heavy systems are not query-ready after corpus growth and require rebuilding, while LENS and ReAct remain query-ready.

\subsection{Cost and Statistical Checks}

Paired McNemar tests on official EM do not show a significant LENS--ReAct answer-quality difference on $D_{500}$ (62.4\% vs. 65.2\%, $p=0.1143$) or fullwiki dev-150 (43.3\% vs. 42.7\%, $p=1.0000$). Query-time budget shows the expected trade-off: on $D_{500}$, ReAct uses 11.8K tokens per query, while LENS uses 16.5K but improves evidence recall by 34.4 pp and grounding by 25.0 pp. Auxiliary JudgeAcc follows EM and is not used as a primary score.

\subsection{Ablation Study}

\begin{table}[t]
\centering
{\footnotesize
\begin{tabular}{lcccc}
\toprule
Configuration & EM & F1 & Ev.Rec & Ground \\
\midrule
LENS (Full)                    & \textbf{43.3} & \textbf{57.5} & 45.0          & 84.0 \\
w/o Multi-signal Prior         & \textbf{43.3} & 55.9          & \textbf{45.4} & \textbf{85.3} \\
w/o Sequential Exploration     & 38.0          & 50.9          & 31.0          & 82.0 \\
\bottomrule
\end{tabular}
}
\caption{Ablation study on the fullwiki evaluation ($n$=150 fixed sample IDs). Ev.Rec is supporting-fact document recall; Ground is the percentage of answers traceable to retrieved evidence.}
\label{tab:ablation}
\end{table}

Table~\ref{tab:ablation} isolates two LENS components on the fixed fullwiki subset. Removing sequential exploration produces the clearest degradation: EM falls from 43.3\% to 38.0\%, F1 from 57.5\% to 50.9\%, and evidence recall from 45.0\% to 31.0\%. Removing the multi-signal prior does not reduce EM on this subset, but it lowers F1 and changes the cost profile.

\section{Discussion and Conclusions}
\label{sec:conclusion}

Across controlled corpora, LENS is the strongest evidence-localization system: on $D_{500}$ it trails ReAct by 2.8 pp EM, but leads by 34.4 pp evidence recall and 25.0 pp grounding. On fullwiki dev-150, the two systems are effectively tied in official EM/F1 while LENS retains stronger grounding. These results support a precise claim: LENS is not an EM-dominant answer generator, but an index-free evidence localization method that makes answers more traceable to current raw sources.

This distinction matters because EM is a narrow string-match measure. It penalizes acceptable paraphrases and aliases, and it can reward a correct answer that is produced from model memory rather than from retrieved evidence. The Closed-Book reference makes this visible: a non-retrieval model already reaches 35.2\% EM on $D_{500}$ and 38.7\% EM on fullwiki. We therefore treat EM/F1 as primary answer-quality metrics, but interpret them together with Ev.Rec and Ground when evaluating systems intended for auditable, source-grounded search.

LENS's advantages appear on three dimensions. First, it localizes evidence more reliably than ReAct across $D_{125}$, $D_{250}$, and $D_{500}$, indicating that the sequential search loop finds supporting documents rather than only plausible answers. Second, it keeps a high grounding rate, which is essential when users need to inspect the source trail. Third, the stale-index arm shows a lifecycle advantage: BM25-RAG and Hybrid-RAG lose 28.0--28.8 pp EM and nearly all evidence recall when a $D_{125}$ index is used on new $D_{250}$ questions, whereas index-free systems remain query-ready over the updated corpus. This is a freshness and cost-structure advantage, not a claim that LENS always improves fresh-corpus EM.

The present evidence is still bounded. HotpotQA emphasizes lookup and comparison over encyclopedic text, so richer document layouts, aggregation intents, table evidence, and warm-reuse behavior remain future work. Within this scope, LENS demonstrates competitive answer quality, substantially stronger source traceability, and immediate query readiness over changing raw documents---the operating point targeted by Budgeted Evidence Localization.

\bibliography{aaai2027}

\input{technical_supplement_preprint_body}

\end{document}

%% file: technical_supplement_preprint_body.tex
\appendix
\renewcommand{\thesection}{\Alph{section}}
\section*{Technical Supplement}

\noindent\textbf{Scope.} The main paper is self-contained. This supplement records the reviewer-critical formal material and the audit trail behind the latest reported experiments: proofs of the two propositions, per-fact localization and stopping details, the $D_{125}/D_{250}/D_{500}$ dynamic raw-corpus protocol, stale-index and lifecycle checks, fullwiki and sample-context diagnostics, ablations, and semantic-judge rescore summaries. Row-level records remain in the code-and-data package; tables here are compact projections of those artifacts.

\section{Notation}
\label{app:notation}

Table~\ref{tab:notation} collects the symbols used in the main paper and in the proofs below. All symbols are as defined in the main paper; nothing here redefines them.

\par\smallskip
\noindent\begin{minipage}{\columnwidth}
\centering
{\footnotesize
\begin{tabular}{lp{0.64\columnwidth}}
\toprule
Symbol & Meaning \\
\midrule
$\mathcal{D}_t$ & raw-document collection at time $t$, $N=|\mathcal{D}_t|$ \\
$\ell_d$ & length of document $d$ in indexable units \\
$\mathcal{E}_t$ & latent evidence space, all windows $(d,s,e)$ \\
$I(q)$ & query intent \\
$\mathcal{D}_{\mathrm{req}}(q,I)$ & requirement set $\{f_1,\ldots,f_K\}$ \\
$Z_j^*$ & minimal sufficient window for fact $f_j$ \\
$\mathbf{Z}^*$ & per-query target $\{Z_1^*,\ldots,Z_K^*\}$ \\
$\pi_{\text{prior}}$ & query-conditioned initial belief \\
$\pi_{\text{file}},\pi_{\text{pos}}$ & file-level and position-level factors \\
$\mathcal{C}_{\mathrm{init}}$ & candidate document set, $m=|\mathcal{C}_{\mathrm{init}}|$ \\
$\mathcal{C}_{\mathrm{search}}$ & candidate window subspace \\
$\mathcal{H}_t$ & observation history $\{(z_i,o_i)\}_{i=1}^t$ \\
$o_i^{(j)}$ & component of $o_i$ pertaining to fact $f_j$ \\
$L$ & aggregate exploration-round allowance \\
$B$ & query budget (calls, tokens, wall-clock) \\
$\gamma(I)$ & intent-modulated stopping factor \\
$E^*,a$ & consolidated evidence set, answer \\
\bottomrule
\end{tabular}}
\captionof{table}{Notation.}
\label{tab:notation}
\end{minipage}
\par\smallskip

\section{Proof of Proposition 1}
\label{app:proof1}

\setcounter{proposition}{0}
\begin{proposition}[Low-cost space compression, restated]
\label{app:prop:compression}
Let $\mathcal{C}_{\mathrm{init}}$ be the finite candidate document set induced by low-cost prior signals. Subsequent evidence exploration is restricted from $\mathcal{E}_t$ to $\mathcal{C}_{\mathrm{search}}=\{(d,s,e)\mid d\in\mathcal{C}_{\mathrm{init}},\,0\le s<e\le \ell_d\}$. Thus, before any iterative oracle budget is consumed, LENS reduces the effective search domain from the full raw-document collection to a query-conditioned subspace.
\end{proposition}

\begin{proof}
The claim has three parts: that the restriction is well defined and strict, that it is quantitatively a reduction, and that it is available before any raw-document observation round is executed.

\emph{(i) The restriction is well defined.} For a document $d$ of length $\ell_d$, the windows of $d$ are the pairs $(s,e)$ with $0\le s<e\le \ell_d$, that is, the $2$-subsets of $\{0,1,\ldots,\ell_d\}$. Hence $d$ contributes
\begin{equation}
\label{eq:windowcount}
w(d)\;=\;\binom{\ell_d+1}{2}\;=\;\frac{\ell_d(\ell_d+1)}{2}
\end{equation}
elements, and
\begin{equation}
|\mathcal{E}_t|=\sum_{d\in\mathcal{D}_t} w(d),
\qquad
|\mathcal{C}_{\mathrm{search}}|=\sum_{d\in\mathcal{C}_{\mathrm{init}}} w(d).
\end{equation}
Because $\mathcal{C}_{\mathrm{init}}\subseteq\mathcal{D}_t$ by construction, every element of $\mathcal{C}_{\mathrm{search}}$ is an element of $\mathcal{E}_t$, so $\mathcal{C}_{\mathrm{search}}\subseteq\mathcal{E}_t$, with strict inclusion whenever $\mathcal{C}_{\mathrm{init}}\subsetneq\mathcal{D}_t$. Both sets are finite since $\ell_d<\infty$ for every $d$.

\emph{(ii) Layer~2 acts only on $\mathcal{C}_{\mathrm{search}}$.} The information-directed objective of the main paper~\cite{russo2014ids} takes its $\arg\min$ over $z\in\mathcal{C}_{\mathrm{search}}$, and the proposal mixture $\pi_t$ that approximates it is a convex combination of $\pi_{\text{lex}},\pi_{\text{local}},\pi_{\text{global}}$, each of which is supported on $\mathcal{C}_{\mathrm{search}}$. A convex combination of measures supported on a set is supported on that set, so $\mathrm{supp}(\pi_t)\subseteq\mathcal{C}_{\mathrm{search}}$ for every $t$. Consequently no proposal, and therefore no observation, is ever drawn from $\mathcal{E}_t\setminus\mathcal{C}_{\mathrm{search}}$: the effective domain of the sequential stage is $\mathcal{C}_{\mathrm{search}}$, not $\mathcal{E}_t$.

\emph{(iii) The reduction precedes raw-document observations.} Write $m=|\mathcal{C}_{\mathrm{init}}|$ and $N=|\mathcal{D}_t|$. Combining the two sums with \eqref{eq:windowcount},
\begin{equation}
\label{eq:ratio}
\frac{|\mathcal{C}_{\mathrm{search}}|}{|\mathcal{E}_t|}
=\frac{\sum_{d\in\mathcal{C}_{\mathrm{init}}}\ell_d(\ell_d+1)}
{\sum_{d\in\mathcal{D}_t}\ell_d(\ell_d+1)} .
\end{equation}
If document lengths lie in $[\ell_{\min},\ell_{\max}]$, bounding the numerator above and the denominator below gives
\begin{equation}
\label{eq:ratiobound}
\frac{|\mathcal{C}_{\mathrm{search}}|}{|\mathcal{E}_t|}
\;\le\;
\frac{m}{N}\cdot\frac{\ell_{\max}(\ell_{\max}+1)}{\ell_{\min}(\ell_{\min}+1)} ,
\end{equation}
which reduces to $m/N$ when document lengths are comparable. It remains to verify that $\mathcal{C}_{\mathrm{init}}$ is fixed before the sequential stage observes candidate document content. By the definition of Layer~1, $\mathcal{C}_{\mathrm{init}}$ is induced by the signal family $\mathcal{K}_0$, whose members are lexical anchors, document-path structure, compiled summaries when already available, historical source-grounded evidence from prior searches, and lightweight corpus scans. Each is computed from the query, the corpus, or persisted artifacts without enumerating $\mathcal{E}_t$ and without asking the LLM to judge candidate document windows. Query-only calls used to parse anchors or requirements are accounted in the constant term of Appendix~\ref{app:proof2}; they do not observe raw-document candidates. Hence the domain reduction in \eqref{eq:ratiobound} is available before the first exploration round, and before the iterative allowance $L$ is touched.
\end{proof}

\begin{corollary}[Magnitude of the reduction in our protocol]
\label{cor:magnitude}
With the file-admission width $m=10$ used in all reported controlled runs and comparable article lengths, \eqref{eq:ratiobound} gives $|\mathcal{C}_{\mathrm{search}}|/|\mathcal{E}_t|\approx 1.8\times10^{-3}$ for $D_{125}$ ($N=5{,}416$), $9.3\times10^{-4}$ for $D_{250}$ ($N=10{,}808$), and $4.7\times10^{-4}$ for $D_{500}$ ($N=21{,}424$). The reduction is roughly three orders of magnitude and improves as the corpus grows, since $m$ is a fixed configuration constant while $N$ increases.
\end{corollary}

\begin{remark}[The price of compression]
\label{rem:ceiling}
Proposition~\ref{app:prop:compression} is a statement about cost, not about sufficiency, and it is worth stating its cost explicitly. All posterior mass in Layer~2 is confined to $\mathcal{C}_{\mathrm{search}}$: if $Z_j^*\notin\mathcal{C}_{\mathrm{search}}$ for some requirement $f_j$, then $P(Z_j^*\mid f_j,q,\mathcal{H}_t)=0$ for all $t$, and no amount of round budget can recover $f_j$. Compression therefore converts a factor-$N/m$ saving into a recall ceiling determined by the quality of $\pi_{\text{file}}$. The \emph{not-retrieved} class of the failure taxonomy isolates this ceiling, while the controlled results show that the sequential loop can still produce higher supporting-fact recall than ReAct and index-heavy RAG under the reported protocol. The bounded self-correction path partially lifts the ceiling by re-entering the loop with an enlarged candidate set, which is why self-correction is a recall mechanism and not merely a synthesis retry.
\end{remark}

\section{Proof of Proposition 2}
\label{app:proof2}

We first fix the accounting convention, then prove the bound.

\begin{definition}[Oracle interaction]
\label{def:oracle}
An \emph{oracle interaction} is one LLM request issued by the search pipeline on behalf of a single query. Requests issued by the evaluation harness --- for instance answer scoring --- are not part of the search pipeline and are excluded, as they are in the reported telemetry.
\end{definition}

\noindent Throughout, $L$ denotes the \emph{aggregate} exploration-round allowance: the total number of rounds executed across the initial loop and the at most one bounded self-correction re-entry. This is the quantity the implementation caps with a shared round counter, so re-entry draws from the same allowance rather than from a fresh one.

\setcounter{proposition}{1}
\begin{proposition}[Bounded oracle complexity, restated]
\label{app:prop:oracle}
For an aggregate loop budget of $L$ exploration rounds, the number of oracle interactions performed by LENS on a query is bounded by $c_0+c_1L$, where $c_0,c_1$ depend only on the configuration and not on the query; in our configuration $c_0=4$ and $c_1=2$. The bound is independent of the number of requirements $K$ and independent of $|\mathcal{E}_t|$.
\end{proposition}

\begin{proof}
Partition the oracle interactions issued on a query by the stage of Algorithm~1 that issues them. The stages are disjoint and exhaustive by construction of the pipeline, so the total is the sum of the per-stage counts.

\emph{Stage S1 (anchor extraction, unconditional).} Prior formation fuses the five families of $\mathcal{K}_0$. Four of them --- path structure, compiled summaries, historical evidence, and corpus scans --- are computed without the LLM, as used in the proof of Proposition~\ref{app:prop:compression}. Lexical anchor extraction issues exactly one request. Count: $1$.

\emph{Stage S2 (requirement decomposition, unconditional).} Deriving $\mathcal{D}_{\mathrm{req}}(q,I)=\{f_1,\ldots,f_K\}$ from $q$ and $I(q)$ is a single request whose output is the entire requirement set, independent of the resulting $K$. Count: $1$.

\emph{Stage S3 (exploration rounds, at most $L$ iterations).} Each round issues exactly two requests: one \emph{proposal-and-observation} request, which reads the raw text of the proposed regions and returns the observation $o_t$; and one \emph{coverage} request, which evaluates the accumulated evidence against the outstanding requirements and returns the pair (complete, missing). No other request is issued inside a round. Count: $2$ per round, hence at most $2L$ in total.

\emph{Stage S4 (answer synthesis, unconditional).} Synthesis over $E^*$ is a single request. Count: $1$.

\emph{Stage S5 (answer-span calibration, conditional).} At most one request. Count: $\le 1$.

Summing, the total number of oracle interactions is at most
\begin{equation}
\label{eq:bound}
\underbrace{1}_{\text{S1}}+\underbrace{1}_{\text{S2}}+\underbrace{2L}_{\text{S3}}+\underbrace{1}_{\text{S4}}+\underbrace{1}_{\text{S5}}
\;=\;4+2L,
\end{equation}
so $c_0=4$ and $c_1=2$. Both constants are properties of the pipeline configuration; neither depends on $q$.

\emph{Independence of $K$.} The only stages whose count could plausibly scale with the number of requirements are S2 and S3. S2 returns the whole set in one request. In S3, the coverage request evaluates all outstanding requirements in a single prompt, and the proposal-and-observation request returns one observation $o_t$ that is scored jointly against all outstanding requirements; by Lemma~\ref{lem:factorize} below, that single observation suffices to update all $K$ per-fact beliefs. Therefore the per-round count is $2$ for every $K\ge1$, and \eqref{eq:bound} contains no factor of $K$.

\emph{Independence of $|\mathcal{E}_t|$ and of $|\mathcal{C}_{\mathrm{search}}|$.} No stage enumerates $\mathcal{E}_t$ or $\mathcal{C}_{\mathrm{search}}$. The exact information-directed $\arg\min$ over $\mathcal{C}_{\mathrm{search}}$ is never computed; it is replaced by sampling from the mixture $\pi_t$, which requires one draw per round irrespective of the cardinality of the support. Hence the right-hand side of \eqref{eq:bound} is free of any term in $N$, $\ell_d$, or $|\mathcal{C}_{\mathrm{search}}|$, and the online oracle cost of a query does not grow with corpus size.
\end{proof}

\begin{remark}[A second, independent bound]
\label{rem:tokenbound}
Equation~\eqref{eq:bound} bounds the call count induced by the control flow. The aggregate token budget supplies an independent stopping condition rather than an additional source of calls: before every request the implementation checks whether the remaining token budget can accommodate that request under the configured per-request context cap. If the check fails, the request is not issued and the loop exits. Thus token budgeting can only reduce the realized count below \eqref{eq:bound}; it cannot make the count grow with $K$ or with corpus size. In all reported runs the call-count bound is the active theoretical constraint, and no execution was recorded as budget-exceeded.
\end{remark}

\begin{remark}[Empirical check, and why the bound is loose]
\label{rem:empirical}
At the configuration used for all controlled dynamic results, $L=3$, so \eqref{eq:bound} evaluates to $4+2\cdot3=10$ oracle interactions per question. The table-level budget summaries for LENS report average oracle calls of 4.00 on $G_{125}$, 4.00 on $G_{250}$, and 4.04 on $G_{500}$, with zero budget-exceeded records. The bound therefore holds with substantial slack. The slack is expected and is not evidence that the accounting is wrong: \eqref{eq:bound} is a worst case over control flow, whereas most questions terminate through a short-circuit --- the coverage check reports completion, or an intent-gated direct-analysis path resolves the query before the loop is exhausted. Proposition~\ref{app:prop:oracle} should be read as a guarantee that the online cost cannot blow up with corpus size or requirement count, not as a prediction of the mean.
\end{remark}

\section{Per-Fact Localization and Stopping}
\label{app:perfact}

The main paper localizes evidence per atomic fact and maintains one belief per fact, updating all of them from a shared observation history. This section records the conditional statement under which the per-fact update is exact and explains why the stopping rule aggregates by the weakest requirement.

An observation returned by the oracle on a proposed region $z_i$ is structured: it reports, for each outstanding requirement, whether $z_i$ supplies that requirement. Write $o_i=(o_i^{(1)},\ldots,o_i^{(K)})$ for its components. The proposed regions $z_1,\ldots,z_t$ are treated as realized actions of the adaptive policy and are conditioned on in the likelihood below; this is the standard conditioning convention for adaptive data collection.

\begin{lemma}[Exact factorization of the joint posterior]
\label{lem:factorize}
Assume
\begin{enumerate}
\item[(A1)] \emph{prior independence across requirements}:
$\pi_{\text{prior}}(\mathbf{Z}^*\mid q,\mathcal{D}_t)=\prod_{j=1}^{K}\pi_{\text{prior}}(Z_j^*\mid f_j,q,\mathcal{D}_t)$; and
\item[(A2)] \emph{component-wise conditional independence of observations given the realized proposals}:
$P(o_i\mid \mathbf{Z}^*,z_i,q,\mathcal{H}_{i-1})=\prod_{j=1}^{K}P\!\left(o_i^{(j)}\mid Z_j^*,z_i,f_j,q\right)$ for every $i$.
\end{enumerate}
Then the joint posterior over $\mathbf{Z}^*$ factorizes across requirements,
\begin{equation}
\label{eq:factorized}
P(\mathbf{Z}^*\mid q,\mathcal{H}_t)=\prod_{j=1}^{K}P\!\left(Z_j^*\mid f_j,q,\mathcal{H}_t\right),
\end{equation}
and each factor is exactly the per-fact belief of the main paper.
\end{lemma}

\begin{proof}
Condition on the realized proposal sequence $z_{1:t}$. By the chain rule,
\begin{equation}
P(\mathbf{Z}^*\mid q,\mathcal{H}_t)\propto
\left[\prod_{i=1}^{t}P(o_i\mid \mathbf{Z}^*,z_i,q,\mathcal{H}_{i-1})\right]\pi_{\text{prior}}(\mathbf{Z}^*\mid q,\mathcal{D}_t),
\end{equation}
where proposal probabilities are not likelihood factors for $\mathbf{Z}^*$ after conditioning on the chosen regions. Substituting (A2) into the product and (A1) into the prior, and exchanging the two finite products,
\begin{align}
P(\mathbf{Z}^*\mid q,\mathcal{H}_t)
&\propto \prod_{i=1}^{t}\prod_{j=1}^{K}P\!\left(o_i^{(j)}\mid Z_j^*,z_i,f_j,q\right)\notag\\
&\quad\times\prod_{j=1}^{K}\pi_{\text{prior}}(Z_j^*\mid f_j,q,\mathcal{D}_t)\notag\\
&=\prod_{j=1}^{K}\left[\prod_{i=1}^{t}P\!\left(o_i^{(j)}\mid Z_j^*,z_i,f_j,q\right)\right.\notag\\
&\qquad\qquad\left.\times\,\pi_{\text{prior}}(Z_j^*\mid f_j,q,\mathcal{D}_t)\right].
\end{align}
Each bracketed term depends on $Z_j^*$ alone. Since the candidate space is finite in the paper's raw-document abstraction, the normalizing sum over $\mathbf{Z}^*$ factorizes into the product of the $K$ per-fact normalizers. The $j$-th normalized factor is therefore precisely $P(Z_j^*\mid f_j,q,\mathcal{H}_t)$, establishing \eqref{eq:factorized}.
\end{proof}

\begin{corollary}[One call, $K$ updates]
\label{cor:onecall}
Under (A1)--(A2), a single oracle interaction on region $z_i$ returns the full vector $o_i$ and therefore supplies the sufficient statistic for updating all $K$ marginals. The per-round oracle count is thus independent of $K$, which is the step used in the proof of Proposition~\ref{app:prop:oracle}.
\end{corollary}

\begin{remark}[Where the assumptions bind]
\label{rem:assumptions}
(A1) fails when requirements are logically coupled, for example when the value of one fact restricts where the other can occur. LENS does not attempt to represent such coupling inside the localization belief: coupling is handled at synthesis time by the transformation rules attached to $\mathcal{D}_{\mathrm{req}}(q,I)$. (A2) is an idealization of a structured LLM judgment; a single oracle response could correlate its per-requirement verdicts. The exactness claim is therefore conditional. Without these assumptions, the per-fact update should be read as a tractable approximation, while the oracle-count statement in Proposition~\ref{app:prop:oracle} still holds as a control-flow bound because the implementation asks for the full vector $o_i$ in one request.
\end{remark}

\paragraph{Stopping by the weakest requirement.} The main paper stops when $\min_{j\le K}\mathrm{GLR}_t^{(j)}$ exceeds $\beta(t,\delta)\,\gamma(I)$. By Definition~2 of the main paper, localization is complete only when every requirement is covered. Let $A_j$ be the event that the currently favoured window $\hat Z_{j,t}$ is the true $Z_j^*$. The event of interest is $\bigcap_{j\le K}A_j$, and a conjunction is certified by its weakest conjunct; the corresponding statistic is therefore $\min_j \mathrm{GLR}_t^{(j)}$. Using $\max_j$ or a mean would allow the loop to halt while some requirement remains unresolved, which is exactly the multi-hop failure mode that a query-level single-window target would hide.

\paragraph{Intent modulation.} Suppose each per-fact test errs with probability at most $\delta_0$. By a union bound,
\begin{equation}
\Pr\!\left[\textstyle\bigcup_{j\le K}\bar A_j\right]\;\le\;\sum_{j\le K}\Pr[\bar A_j]\;\le\;K\delta_0 .
\end{equation}
To attain an overall failure probability at most $\delta$, it is sufficient to set $\delta_0=\delta/K$, so the per-fact threshold should be at least $\beta(t,\delta/K)$. The multiplicative factor $\gamma(I)\ge1$ is chosen as a conservative intent-level surrogate satisfying
\begin{equation}
\beta(t,\delta)\,\gamma(I)\;\ge\;\beta(t,\delta/K(I))
\end{equation}
for the requirement count induced by the intent class. Lookup intents have $K=1$ and can use $\gamma=1$; comparison, computation, and aggregation intents induce $K>1$ and therefore require a tighter criterion. In implementation, $\mathrm{GLR}_t^{(j)}$ is conceptual because LLM likelihoods are unavailable in closed form; the realized policy halts when the coverage check reports all requirements satisfied or when $L$ is exhausted. The analysis explains the shape of the stopping rule rather than claiming calibrated numerical confidence.

\section{Dynamic Evaluation Protocol}
\label{app:protocol}

The protocol is designed to make every comparison auditable: systems answer identical frozen questions over byte-identical corpus snapshots, and every stage is bound by sample-ID, frozen-order, and corpus checksums. The latest reported controlled evaluation uses nested question sets $G_{125}\subset G_{250}\subset G_{500}$ and matched snapshots $D_{125}\subset D_{250}\subset D_{500}$.

\paragraph{Sampling.} The population is the HotpotQA fullwiki validation split with $7{,}405$ questions. Sampling uses the eight cells of $\mathit{type}\times\mathit{supporting\text{-}fact\ count}$, seed $42$, and a stratum-balanced frozen order. $G_{125}$ and $G_{250}$ are prefixes of the frozen $G_{500}$ parent order. The realized maximum absolute deviation from the population distribution is $0.70$ percentage points for $G_{125}$ and $0.30$ percentage points for $G_{250}$, with no empty stratum. Full sample IDs and hashes are shipped with the artifact package; Table~\ref{tab:binding_audit} reports short hash prefixes for readability.

\par\smallskip
\noindent\begin{minipage}{\columnwidth}
\centering
{\footnotesize
\begin{tabular}{lrrr}
\toprule
Stage & Samples & Sample hash & Order hash \\
\midrule
$G_{125}$ & 125 & a54c27b0 & 7c0bad56 \\
$G_{250}$ & 250 & 0b98ccda & 69ae9f6c \\
$G_{500}$ & 500 & 4fc0115e & b9788e58 \\
\bottomrule
\end{tabular}}
\captionof{table}{Frozen binding audit. Hashes are shortened to eight characters; the package records the full sample-ID and frozen-order checksums.}
\label{tab:binding_audit}
\end{minipage}
\par\smallskip

\paragraph{Corpus construction.} Each snapshot contains the supporting evidence articles for its questions, the context distractors shipped with those questions, and a deterministic background pool selected at a $3{:}1$ ratio to question-linked articles. Documents are materialized as one extensionless raw-text file per article. This article-level materialization is necessary for the raw-corpus protocol: shard-level materialization would make the apparent corpus size depend on dump packaging rather than on evidence-bearing articles.

\par\smallskip
\noindent\begin{minipage}{\columnwidth}
\centering
{\footnotesize
\setlength{\tabcolsep}{1mm}
\begin{tabular}{lrrrrr}
\toprule
Snapshot & $n$ & Articles & Ev. & Dist. & Bg. \\
\midrule
$D_{125}$ & 125 & 5{,}416 & 250 & 1{,}104 & 4{,}062 \\
$D_{250}$ & 250 & 10{,}808 & 500 & 2{,}202 & 8{,}106 \\
$D_{500}$ & 500 & 21{,}424 & 996 & 4{,}360 & 16{,}068 \\
\bottomrule
\end{tabular}}
\captionof{table}{Snapshot composition. Ev., Dist., and Bg. are evidence, context-distractor, and background articles. The $D_{500}$ evidence count is below $2n$ because some questions share evidence articles. Full corpus checksums are recorded in the artifact package.}
\label{tab:snapshot_composition}
\end{minipage}
\par\smallskip

The snapshots satisfy append-only growth: no existing article is edited or removed between stages. The corpora are not oracle corpora; distractors and background articles outnumber evidence articles, and publication-ready tables are emitted only when all compared systems share the same sample-ID, frozen-order, and corpus binding.

\section{Experimental Result and Lifecycle Audit}
\label{app:results_audit}

\par\smallskip
\noindent\begin{minipage}{\columnwidth}
\centering
{\footnotesize
\setlength{\tabcolsep}{1mm}
\begin{tabular}{llrr}
\toprule
Stage & System & EM & Ev.Rec \\
\midrule
$D_{125}$ & LENS & 64.0 & 89.1 \\
$D_{125}$ & ReAct Search & 61.6 & 46.1 \\
$D_{125}$ & Hybrid-RAG & 36.0 & 81.3 \\
$D_{125}$ & BM25-RAG & 27.2 & 71.6 \\
$D_{125}$ & Closed-Book & 34.4 & 0.0 \\
$D_{250}$ & LENS & 66.4 & 85.9 \\
$D_{250}$ & ReAct Search & 67.6 & 47.9 \\
$D_{250}$ & Hybrid-RAG & 38.8 & 77.8 \\
$D_{250}$ & BM25-RAG & 30.4 & 71.1 \\
$D_{250}$ & Closed-Book & 33.6 & 0.0 \\
$D_{500}$ & LENS & 62.4 & 84.8 \\
$D_{500}$ & ReAct Search & 65.2 & 50.4 \\
$D_{500}$ & Hybrid-RAG & 38.4 & 80.8 \\
$D_{500}$ & BM25-RAG & 28.8 & 71.8 \\
$D_{500}$ & Closed-Book & 35.2 & 0.0 \\
\bottomrule
\end{tabular}}
\captionof{table}{Controlled-scale quality audit (\%). Ev.Rec is supporting-fact document recall. Closed-Book is included as a no-retrieval reference and therefore has no evidence recall.}
\label{tab:controlled_audit}
\end{minipage}
\par\smallskip

Table~\ref{tab:controlled_audit} is the compact supplement view of the controlled dynamic results. It shows the main pattern behind the paper's interpretation: ReAct is competitive or stronger on exact-match answer quality, while LENS consistently gives the strongest supporting-fact localization across all three controlled scales.

\par\smallskip
\noindent\begin{minipage}{\columnwidth}
\centering
{\footnotesize
\setlength{\tabcolsep}{1mm}
\begin{tabular}{lccrrrr}
\toprule
System & Rdy & Rbld & Idx(s) & Store & Tok/Q & Lat(s) \\
\midrule
LENS & yes & no & 0.0 & 0 & 16.5K & 68.9 \\
ReAct Search & yes & no & 0.0 & 0 & 11.8K & 54.7 \\
BM25-RAG & no & yes & 4.0 & 10.0MB & 1.2K & 11.3 \\
Hybrid-RAG & no & yes & 5.7 & 54.9MB & 1.4K & 14.3 \\
\bottomrule
\end{tabular}}
\captionof{table}{$D_{500}$ lifecycle and query-budget audit. Rdy is immediate query readiness after corpus materialization, Rbld is whether a persistent index rebuild is required, Idx is index build time, Store is index storage, Tok/Q is tokens per query, and Lat is query latency.}
\label{tab:lifecycle_audit}
\end{minipage}
\par\smallskip

\par\smallskip
\noindent\begin{minipage}{\columnwidth}
\centering
{\footnotesize
\begin{tabular}{llrrr}
\toprule
System & Metric & Fresh & Stale & Gap \\
\midrule
BM25-RAG & EM & 30.4 & 2.4 & 28.0 \\
BM25-RAG & Ev.Rec & 70.1 & 0.0 & 70.1 \\
Hybrid-RAG & EM & 42.4 & 13.6 & 28.8 \\
Hybrid-RAG & Ev.Rec & 75.1 & 5.5 & 69.6 \\
ReAct Search & EM & 72.0 & 72.8 & -0.8 \\
ReAct Search & Ev.Rec & 54.1 & 48.3 & 5.7 \\
LENS & EM & 69.6 & 67.2 & 2.4 \\
LENS & Ev.Rec & 84.7 & 83.9 & 0.8 \\
\bottomrule
\end{tabular}}
\captionof{table}{Stale-index audit for the $D_{125}\to D_{250}$ transition. Index-heavy systems reuse a $D_{125}$ index on newly added $D_{250}$ questions; index-free systems query the updated corpus directly.}
\label{tab:stale_audit}
\end{minipage}
\par\smallskip

The lifecycle and staleness tables separate two claims. Table~\ref{tab:lifecycle_audit} measures readiness and cost structure on the fresh $D_{500}$ corpus. Table~\ref{tab:stale_audit} measures freshness under corpus growth. The latter is not a fresh-corpus quality comparison: it asks whether a system remains usable when its pre-materialized index is stale.

\par\smallskip
\noindent\begin{minipage}{\columnwidth}
\centering
{\footnotesize
\setlength{\tabcolsep}{1mm}
\begin{tabular}{lrrrrr}
\toprule
System & EM & F1 & Ev.Rec & Ground & Calls \\
\midrule
\multicolumn{6}{l}{\emph{Fullwiki dev-150 (raw dump, zero indexing)}} \\
Closed-Book & 38.7 & 49.7 & 0.0 & 0.0 & 0.00 \\
ReAct Search & 42.7 & 57.3 & 50.4 & 70.7 & 7.01 \\
\midrule
\multicolumn{6}{l}{\emph{Sample-context dev-150 (diagnostic control)}} \\
ReAct Search & 54.7 & 69.9 & 37.2 & 63.3 & 4.94 \\
LENS & 52.7 & 67.6 & 78.2 & 86.7 & 4.00 \\
\bottomrule
\end{tabular}}
\captionof{table}{Fullwiki and sample-context diagnostics on the fixed dev-150 sample IDs. Calls is average oracle calls per question. The sample-context rows are diagnostic controls and are not presented as the primary raw-corpus result.}
\label{tab:fullwiki_control_audit}
\end{minipage}
\par\smallskip

\par\smallskip
\noindent\begin{minipage}{\columnwidth}
\centering
{\footnotesize
\setlength{\tabcolsep}{1mm}
\begin{tabular}{lrrrr}
\toprule
Configuration & EM & F1 & Ev.Rec & Ground \\
\midrule
LENS full & 43.3 & 57.5 & 45.0 & 84.0 \\
w/o Multi-signal Prior & 43.3 & 55.9 & 45.4 & 85.3 \\
w/o Sequential Exploration & 38.0 & 50.9 & 31.0 & 82.0 \\
\bottomrule
\end{tabular}}
\captionof{table}{Fullwiki ablation audit on the fixed dev-150 sample IDs. The full LENS row is the main-paper reference configuration; the two ablation rows are the matched ablation run. Sequential exploration is the main contributor to answer quality and evidence recall.}
\label{tab:ablation_audit}
\end{minipage}
\par\smallskip

\par\smallskip
\noindent\begin{minipage}{\columnwidth}
\centering
{\footnotesize
\begin{tabular}{lrr}
\toprule
System ($D_{500}$) & Official EM & JudgeAcc \\
\midrule
ReAct Search & 65.2 & 87.0 \\
LENS & 62.4 & 85.2 \\
BM25-RAG & 28.8 & 64.0 \\
Hybrid-RAG & 38.4 & 61.6 \\
Closed-Book & 35.2 & 50.8 \\
\bottomrule
\end{tabular}}
\captionof{table}{Auxiliary semantic-judge audit on $D_{500}$. JudgeAcc is used as a diagnostic only; official EM/F1 remain the primary answer-quality metrics.}
\label{tab:judge_audit}
\end{minipage}
\par\smallskip

\section{Prompt Interfaces, Taxonomy, and Extra Analyses}
\label{app:audit}

\paragraph{Prompt interfaces.} For page economy, the supplement specifies prompt interfaces rather than duplicating long verbatim templates. The executable templates in the code-and-data package implement the same typed interfaces: requirement decomposition maps $(q,I(q))$ to $\{f_1,\ldots,f_K\}$; proposal-and-observation maps a candidate region and outstanding requirements to the vector $o_i$; coverage checking maps accumulated evidence to a complete/missing decision; synthesis maps $E^*$ and transformation rules to the answer; answer-span calibration optionally normalizes the final span. The evaluation judge prompt is separated from system prompts and uses a dedicated judge configuration, so answer scoring is not counted in the search oracle budget of Definition~\ref{def:oracle}.

\paragraph{Failure taxonomy.} Each question is assigned to exactly one class in the following order. First, a question is \emph{correct} if official exact match succeeds. Second, among incorrect cases, it is \emph{retrieved-but-unanswered} if a gold evidence document or sentence was retrieved but no answer was produced. Third, it is \emph{answered-but-wrong} if sufficient gold evidence was retrieved but the produced answer is not exact-match correct. Remaining cases are \emph{not-retrieved}. This order is mutually exclusive and exhaustive, and it separates discovery failures from synthesis failures. The connection to Remark~\ref{rem:ceiling} is direct: not-retrieved cases are the observable footprint of the recall ceiling induced by file-level compression, whereas answered-but-wrong cases occur after discovery has succeeded.

\paragraph{Artifact consistency.} The dynamic tables are projections of the same run family: $D_{125}$, $D_{250}$, and $D_{500}$ share the nested frozen-order protocol, and the stale-index arm records the from-corpus and to-corpus checksums for each row. Fullwiki and sample-context dev-150 diagnostics share the fixed sample-ID checksum, making their differences attributable to corpus scope rather than to sample changes. The ablation table keeps the main-paper full LENS row as the reference configuration and uses the matched ablation artifact for the two ablated rows. Additional ablation or sensitivity tables should be included only when their frozen artifacts share the same binding checksums as the main runs.